%% file: main_combo.tex
\documentclass[a4paper]{article}

\usepackage{graphicx}
\usepackage{subfigure} 
\usepackage{fancyhdr}
\usepackage{fancybox}
\usepackage{natbib}
\usepackage{hyperref}
\usepackage{xcolor}

\usepackage{algorithm}
\usepackage{algorithmic}

\input{entete}

\begin{document} 

\lhead{S. Ko\c co, C. Capponi} 
\rhead{On multi-class learning through the minimization of the confusion matrix norm}
\rfoot[Technical Report V 1.0]{\thepage} 
\cfoot{} 
\lfoot[\thepage]{Technical Report V 2.0}

\renewcommand{\headrulewidth}{0.4pt}  
\renewcommand{\footrulewidth}{0.4pt}

\title{On multi-class learning through the minimization of the confusion matrix norm}

 \author{ Sokol Ko\c co 
\and C\'ecile Capponi
\and  Aix-Marseille Univ., LIF-QARMA, CNRS, UMR 7279, F-13013, Marseille, France\\ \textit{\{firstname.name\}@lif.univ-mrs.fr}}

\maketitle

\input{abstract}


\input{introduction}

\input{notations}

\input{combo}

\input{theory}

\input{results}

\input{conclusion}

\bibliography{biblio}
\bibliographystyle{icml2013}

\end{document}

%% file: entete.tex
\usepackage{amsmath}
\usepackage{amssymb}
\usepackage{amsthm}
\usepackage{bbm}

\newcommand{\xbf}{\ensuremath{\mathbf{x}}}

\newcommand{\argmax}{\operatorname{argmax}}

\newtheorem{corollary}{Corollary}

\newtheorem{definition}{Definition}
\newtheorem{lemma}{Lemma}
\newtheorem{theorem}{Theorem}

%% file: abstract.tex
\begin{abstract}

In imbalanced multi-class classification problems, the misclassification rate as an error measure may not be a relevant choice.
Several methods have been developed where the performance measure retained richer information than the mere misclassification rate:  
misclassification costs, ROC-based information, etc.
Following this idea of dealing with alternate measures of performance, we propose to address imbalanced classification problems by using a new measure to be optimized: 
the norm of the confusion matrix.

Indeed, recent results show that using the norm of the confusion matrix as an error measure can be quite interesting due to the fine-grain informations contained in the matrix, 
especially in the case of imbalanced classes.
Our first contribution then consists in showing that optimizing criterion based on the confusion matrix gives rise to a common background for cost-sensitive methods 
aimed at dealing with imbalanced classes learning problems.
As our second contribution, we propose an extension of a recent multi-class boosting method --- namely AdaBoost.MM --- to the imbalanced class problem, by greedily minimizing  the empirical norm of the confusion matrix.
A theoretical analysis of the properties of the proposed method is presented, while experimental results illustrate the behavior of the algorithm and show the relevancy of the
approach compared to other methods.

\end{abstract}

%% file: introduction.tex
\section{Introdution}
\label{sec:intro}

In multi-class classification, learning from imbalanced data concerns theory and algorithms that process a relevant learning task whenever data is not uniformly distributed
among classes. 
When facing imbalanced classes, the classification accuracy is not a fair measure to be optimized \citep{Fawcett06introductionROC}. 
Accuracy may be artificially quite high 
in case of extremely imbalanced data: majority classes are favored, while minority classes are not recognized. Prediction is biased toward
the classes with the highest priors. Such a bias gets stronger within the multi-class setting.

In the binary setting, learning from imbalanced data has been well-studied over the past years, leading to many algorithms and theoretical
results \citep{He09learning}. It is mostly achieved by either 
resampling methods for rebalancing the data (e.g. \cite{Estabrooks04multiple}), 
and/or by using cost-sensitive methods (e.g. \cite{Ting00comparative}),
or with additional assumptions (e.g. active learning within kernel-based methods \cite{Bordes05fast}). 

However, learning from imbalanced data within a multi-class or multi-label setting is still an open research problem, which is sometimes adressed through the study of
some specific measures of interest. Most of the time, generalizing the binary setting to the multi-class setting 
is based on the one-vs-all (or one-vs-one) usual trade-off \citep{Abe04aniterative}. 
Recently, \cite{Zhou10onmulticlass} proposed a well-founded general rescaling approach of multi-class cost-sensitive learning, that pointed out the need of separating misclassification
costs from class imbalance rates during the learning algorithm.

It is worth noticing that specific learning tasks other than classification have been addressed through the optimization
of relevant measures within the multi-class imbalanced setting \citep{Chapelle11future,Yue07supportvector,Wang2012software,Tang11towards}.
Although related to accuracy, these measures are intended to better model what one would expect to be a a relevant performance measure in such a setting. 
Meanwhile, the correlations between 
these alternative measures and accuracy have been partly studied \citep{Cortes03advances} without any theoretical result so far \citep{He09learning}.

Cost-sensitive methods are usually based on a cost matrix which embeds misclassification costs carrying various meanings. 
These methods weight the error measure according to class-based costs of each misclassification, 
computed from the confusion matrix results \citep{Elkan01thefoundations}.
Indeed, the {\em confusion matrix} is one of the most informative performance measures a multi-class learning system can rely on. 
Among other information, it contains how much the classifier is accurate on one class, 
and the way it tends to mistake each class for other ones (confusion among classes).

In the binary classification setting, cost-sensitive approaches provide algorithms that somehow optimize a metric 
computed from the raw confusion matrix where the entries of a row sum up to the number of examples of the class corresponding to the row.
Basically, a raw confusion matrix is a square matrix that represents the count of a classifier's class predictions with respect to the actual outcome on some labeled learning set. 
Computed from the raw matrix, the {\em probabilistic confusion matrix} (section \ref{sec:combo}) exhibits an interesting property: the entries of a row sum up to 1, 
independently from the actual number of examples of the class corresponding to the row. 
With such property, the confusion matrix constitutes a great equalizer tool that can be used to overcome the class-imbalance problem.
Moreover, if one only considers the non-diagonal elements in the matrix, then summing over a row gets 
quite informing about how the corresponding class is correctly handled by the predictor at hand. This paper explores
one way to capitalize on this property:  we advocate that directly minimizing the norm of the confusion matrix is helpful for smoothing the accuracy among imbalanced 
classes, thus giving more importance to minority classes.

Based on recent works on the confusion matrix \citep{Morvant2012,Ralaivola2012}
and helped by a multi-class classification theoretical setting \citep{Mukherjee2010}, the aim of this paper
is to sketch up a computationally and theoretically consistent classification algorithm dubbed CoMBo (section \ref{sec:theory}) 
that is ensured to minimize the norm of the confusion matrix; while also minimizing the classification error as proven in section~\ref{sec:combo}. 
Based on Adaboost, this algorithm greedily processes a sort-of regularization on imbalanced classes, 
in such a way that poorly represented classes are performed as well as majority classes within the overall learning process, 
independently from any prior misclassification cost.  
Section \ref{sec:expe} summarizes the experimental performances of this algorithm, compared to
Adaboost.MM \citep{Mukherjee2010} and other boosting-based related approaches.
Sections \ref{sec:discussion} and \ref{sec:conclusion} propose a discussion on the contributions of this paper and the future works.

%% file: notations.tex
\section{General framework and notation}
\label{sec:notations}

\subsection{General notation}

Matrices are denoted with bold capital letters like $\mathbf{C}$, while vectors are denoted by bold small letters like $\mathbf{x}$.
The entry of the  {\it l}-th row and the {\it j}-th column of $\mathbf{C}$ is denoted $\mathbf{C}(l,j)$, or simply $\mathbf{c}_{l,j} $. 
$\lambda_{max}(\mathbf{C})$ and $Tr(\mathbf{C})$ respectively correspond to the largest eigenvalue and the trace of $\mathbf{C}$. 
The {\it spectral} or {\it operator norm} $\|\mathbf{C}\|$ of $\mathbf{C}$ is defined as:
$$
\|\mathbf{C}\| \stackrel{\hbox{\scriptsize def}}{=} \underset{\mathbf{v}\neq 0}{\max}\frac{\|\mathbf{C} \mathbf{v}\|_2}{\| \mathbf{v}\|_2} \stackrel{\hbox{\scriptsize def}}{=} \sqrt{\lambda_{max}(\mathbf{C^*C})},
$$
where $\|\cdot\|_2$ is the Euclidian norm and $\mathbf{C^*}$ is the conjugate transpose of $\mathbf{C}$.
Let $\mathbf{A}$ and $\mathbf{B}$ be two matrices, then $\mathbf{AB}$ and  $\mathbf{A \cdot B}$ respectively refer to the inner product and the Frobenius inner product of $\mathbf{A}$ and $\mathbf{B}$.

The indicator function is denoted by $\mathbb{I}$; $K$ is the number of classes, $m$ the number of examples and $m_y$ is the number of examples of class $y$, where $y\in \{1,...,K\}$.
$(\xbf_i,y_i), \xbf_i$ or simply $i$ are interchangeably used to denote the $i^\text{th}$ training example.

\subsection{Multi-class boosting framework}

In this paper we use the boosting framework for multi-class classification introduced in \cite{Mukherjee2010}, and more precisely the one defined for AdaBoost.MM.
Algorithms based on the AdaBoost family maintain a distribution over the training samples in order to identify hard-to-classify examples: the greater the weight of an example, the greater the effort devoted to classify these data correctly.
In the considered setting, the distribution over the training examples is replaced by a cost matrix.
Let $S = \{(\xbf_i,y,_i)\}_{i=1}^m$ be a training sample, where $\xbf_i \in X$ and $y_i \in \{1,...,K\}$.
The cost matrix $\mathbf{D}\in \mathbb{R}^{m\times K}$ is constructed so that for a given example $(\xbf_i,y_i)$, $\forall l\neq y_i$: \ $\mathbf{D}(i,y_i)\leq \mathbf{D}(i,l)$, where $i$ is the row of $\mathbf{D}$ corresponding to $(\xbf_i,y_i)$.

This cost matrix is a particular case of cost matrices used in cost sensitive methods (for example, \cite{Sun2006boosting}), where classification costs are given for each example and each class.
However, contrary to those methods, the matrix is not given prior to the learning process, but it is updated after each iteration of AdaBoost.MM so that the misclassification cost reflects the difficulty of correctly classifying an example.
That is, the costs are increased for examples that are hard to classify, and they are decreased for easier ones.

In the case of AdaBoost.MM, the cost matrix $\mathbf{D}$ at iteration $T$, is defined as follows:
$$
\mathbf{D}_T(i,l) \stackrel{\hbox{\scriptsize def}}{=} \left\{
\begin{array}{ll}
\exp(f_T(i,l)-f_T(i,y_i))&\textrm{if } l\neq y_i\\
\\
-\underset{j\neq y_i}{\sum}\exp(f_T(i,j)-f_T(i,y_i))& \textrm{otherwise},
\end{array}\right.
$$
where $f_T(i,l)$ is the score function computed as:
$
f_T(i,l) = \sum_{t=1}^T \alpha_t\mathbb{I} (h_t(i)=l).
$

At each iteration $t$, AdaBoost.MM selects the classifier $h$ and its weight $\alpha$ that minimize the exponential loss:
$$
h_t, \alpha_t = \underset{h,\alpha}{\operatorname{argmin}} \sum_{i=1}^m\underset{j\neq y_i}{\sum}e^{f_{t}(i,j)-f_{t}(i,y_i)} \mbox{, where } f_t(i,j) = \sum_{s=1}^{t-1} \alpha_s\mathbb{I} (h_s(i)=j)+\alpha\mathbb{I}(h(i)=j).
$$

Finally, the output hypothesis of AdaBoost.MM is a simple majority vote: 
$$
H(x) = \underset{l = 1\ldots K}{\operatorname{argmax}}\ f_T(i,l).
$$

%% file: combo.tex
\section{Boosting the confusion matrix}
\label{sec:combo}

\subsection{The confusion matrix as an error measure}

When building predictors, most Empirical Risk Minimization based methods optimize loss functions defined over a training sample and depending on the number of misclassified data.
This reveals to be a poor strategy when tackling problems with class-imbalance, since minimizing the estimated error may result in overfitting the majority classes.

A common tool used for estimating the goodness of a classifier in the imbalanced classes scenario is the confusion matrix.
We give here the probabilistic definitions of the {\it true confusion matrix} and the {\it empirical confusion matrix}.

\begin{definition}[True and empirical confusion matrices]
The true confusion matrix $\mathbf{A} = (a_{l,j})\in \mathbb{R}^{K\times K}$  of a classifier $h: X\leftrightarrow \{1,\hdots, K\}$ over a distribution $\mathcal{D}$ is defined as follows:
\vspace{-0.2cm}
$$
\forall l,j \in \{1,...,K\}, \mathbf{a}_{l,j} \stackrel{\hbox{\scriptsize def}}{=}\ \mathbb{E}_{\xbf|y=l}\mathbb{I}\big(h(\xbf)=j\big) = \mathbb{P}_{(\xbf,y)\sim \mathfrak{D}}( h(\xbf)=j | y=l ).
$$
For a given classifier $h$ and a sample 
$S = \{(\xbf_i,y_i)\}_{i=1}^m \sim \mathcal{D}$, the empirical confusion matrix $\mathbf{A}_S = (\hat{a}_{l,j})\in \mathbb{R}^{K\times K}$  of $h$ is defined as :
\vspace{-0.3cm}
$$
\forall l,j \in \{1,...,K\}, \ \hat{\mathbf{a}}_{l,j} \stackrel{\hbox{\scriptsize def}}{=} 
\sum_{i=1}^m\frac{1}{m_l}\mathbb{I}(h(\xbf_i)=j)\mathbb{I}(y_i=l).
\vspace{-0.3cm}
$$
\end{definition}

An interesting property of these matrices is that the entries of a row sum up to 1, independently from the number of examples contained in the corresponding class.
Moreover, the diagonal entries represent the capability of the classifier $h$ to recognize the different classes, while the non-diagonal entries represent the mistakes of the classifier.

In this paper, and in the following sections, we propose a framework which makes use of the confusion matrix as an error measure in order to learn predictors for the imbalanced classes setting.
As such, we need to redefine the confusion matrices, so that only the mistakes of the classifier are taken into consideration.
This is done by keeping the non-diagonal entries and zeroing the diagonal ones:
\begin{definition}[Error-focused true and empirical confusion matrices]
\label{def:mc}
For a given classifier $h$, we define the empirical and true confusion matrices of $h$ by respectively  $\mathbf{C}_S\!=\!(\hat{\mathbf{c}}_{l,j})_{1\leq l,j \leq K}$ and $\mathbf{C}\!=\!(\mathbf{c}_{l,j})_{1\leq l,j \leq K}$ such that for all $(l,j)$: 

\begin{equation} \label{eq:MC_empirique}
 \hat{\mathbf{c}}_{l,j} \stackrel{\hbox{\scriptsize def}}{=} \left\{
\begin{array}{ll}
0&\textrm{if } l=j\\
\displaystyle\mathbf{\hat{a}}_{l,j}& \textrm{otherwise},
\end{array}\right.
\mathbf{c}_{l,j} \stackrel{\hbox{\scriptsize def}}{=} \left\{
\begin{array}{ll}
0&\textrm{if } l=j\\
\mathbf{a}_{l,j} &\textrm{otherwise}.
\end{array}\right.
\end{equation}

\end{definition}

Let $\mathbf{p} = [P(y=1), ... ,P(y=K)]$ be the vector of class priors distribution, then we have:
\begin{equation}
\label{eq:norml1}
R(h) \stackrel{\hbox{\scriptsize def}}{=} P_{(x,y)\sim \mathcal{D}}(h(x)\neq y) = \|\mathbf{p}\mathbf{C} \|_1,
\end{equation}
where $\|\cdot\|_1$ is the $l1$-norm.
This result means that it is possible to retrieve the true error rate of $h$ from its confusion matrix.
Using this result and the equivalence between norms in finite dimensions, we obtain equation (\ref{eq:norml2}) between the operator norm and the true risk:
\begin{equation}
\label{eq:norml2}
R(h) \leq \sqrt{K}\|\mathbf{C} \|
\end{equation}

Equations (\ref{eq:norml1}) and (\ref{eq:norml2}) together imply that minimizing the norm of the confusion matrix can be a good strategy for dealing with imbalanced classes, 
while at the same time it would ensure a small risk.
Our aim is thus to find a classifier $\hat{h}$ that verifies the following criterion:

\begin{equation}
\label{eq:goal}
\hat{h} = \underset{h\in \mathcal{H}}{\operatorname{argmin}}\|\mathbf{C} \|
\end{equation}

\subsection{Bounding the confusion matrix}
\label{ssec:bcmatrix}

The result given in equation (\ref{eq:goal}) minimizes the operator norm of the true confusion matrix, but it is difficult to use in a practical case, since the underlying distribution $\mathcal{D}$ is unknown.
A popular way to overcome this difficulty is to use the empirical estimation of the confusion matrix.
Theorem 1 in \cite{Ralaivola2012} gives the relation between the {\it true} and the {\it estimated norm} of the confusion matrix.
\footnote{In the original paper, the theorem is applied to the online setting, but it holds also for the batch setting.}
We recall here a particular formulation of this theorem applied to the supervised setting, where the considered loss is the indicator function $\mathbb{I}$.

\begin{corollary}
\label{cor:the}
For any $\delta \in (0;1]$, it holds with probability $1-\delta$ over a sample $S_{(\mathbf{x},y)\sim \mathcal{D}}$ that:
$$
\|\mathbf{C}\| \leq \|\mathbf{C}_{S}\| + \sqrt{2K\sum_{k=1}^K\frac{1}{m_k}\log{\frac{K}{\delta}}},
$$
where $\mathbf{C}_{S}$ is the empirical confusion matrix computed for a classifier $h$ over $S$. 
\end{corollary}

The corollary suggests that our goal may boil down to minimizing the empirical norm of the confusion matrix, which is fairly similar to other optimization problems, where empirical error measures are considered.
Unfortunately, due to the nature of the considered confusion matrix, finding an analytical expression for the norm reveals to be quite challenging.
This is why we propose to optimize an upper bound of $\|\mathbf{C}_{S}\|^2$ instead.

\begin{equation}
\label{eq:nor}
\|\mathbf{C}_{S}\|^2  =  \lambda_{max}(\mathbf{C}_{S}^*\mathbf{C}_{S}) \leq Tr(\mathbf{C}_{S}^*\mathbf{C}_{S})
\end{equation}

The matrix $\mathbf{C}_S^*\mathbf{C}_S$ is positive semi-definite, meaning that all its eigenvalues are positive.
Thus we choose to upper-bound the norm of the confusion matrix by the trace of $\mathbf{C}_S^*\mathbf{C}_S$ as in equation (\ref{eq:nor}).
We can now focus on the value of $Tr(\mathbf{C}_S^*\mathbf{C}_S)$.
\begin{align}
\label{eq:btrace}
Tr(\mathbf{C}_S^*\mathbf{C}_S) = \sum_{l=1}^K \mathbf{C}_S^*\mathbf{C}_S(l,l) 
\leq \sum_{l=1}^K \sum_{j\neq l} \hat{\mathbf{c}}_{l,j} = \underset{l=1}{\overset{K}{\sum}} \underset{{j\neq l}}{\sum} \frac{1}{m_l} \sum_{i=1}^{m} \mathbb{I}(y_i = l)\mathbb{I}(H(i)=j)
\end{align}

Most of the previous equalities are simple rewritings of the involved terms, while the inequality comes from the  fact that the entries of the confusion matrix $\mathbf{C}_S$ are at most 1.
We now have an upper bound on the norm that depends only on the entries of the confusion matrix (equation (\ref{eq:btrace})).

The drawback of this bound is the presence of the indicator function, which is not optimization friendly.
One way to handle this difficulty is to replace the indicator function with loss functions defined over the two classes of indices $l$ and $j$ in the bound.

For a given classifier $h$ and a sample $S$, let $\mathbf{\ell}_{l,j}(h,\xbf)$  be the loss of $h$ choosing class $j$ instead of class $l$ for the example $\xbf$, such that $\forall (\xbf,y)\in S,1\leq l,j\leq K$ and $0\leq \mathbb{I}(h(\xbf)\neq y)\leq \ell_{l,j}(h,\xbf) $. The bound on the confusion matrix can be now expressed using loss functions:
\begin{equation}
\label{eq:optimizationMC}
Tr(\mathbf{C}_S^*\mathbf{C}_S) \leq \underset{l=1}{\overset{K}{\sum}} \underset{{j\neq l}}{\sum} \frac{1}{m_l} \sum_{i=1}^{m} \mathbb{I}(y_i = l)\ell_{l,j}(h,\xbf_i) = \underset{i=1}{\overset{m}{\sum}} \underset{j\neq y_i}{\sum}\frac{1}{m_{y_i}} \ell_{y_i,j}(h,\xbf_i).
\end{equation}

The resulting bound is a sum over two penalization terms for all the learning examples: the first term is $\frac{1}{m_{y_i}}$ and the second one represents the newly introduced loss function.
The first term is often used in the imbalanced classes problems in order to simulate the resampling effect over the training sample.
Indeed, by multiplying the weight of an example with the inverse of the number of examples having the same class, the distribution over the classes becomes more uniform, 
thus promoting all the classes the same way.
Resampling over the examples achieves the same effect, since it aims to retain the same number of examples for each class.

The second term is  common in the imbalanced classes setting, where cost-sensitive methods --- either per-class based or per-couple of classes based --- have been developed.
In these methods this term depends on some misclassification cost computed, or simply given, {\em prior} to the learning phase.

Based on these observations, it follows that  the last term of the bound given in equation \ref{eq:optimizationMC} is quite similar to usual terms in optimization problems for the imbalanced classes setting.
As such it can be seen as a common base for most imbalanced classes learning problems, since it encompasses both resampling and penalization, as advocated in \cite{Zhou10onmulticlass}.
Motivated by these results, and as an intermediate conclusion, we strongly think that minimizing the operator norm of the confusion matrix defined as in equation \ref{eq:MC_empirique}, yields a good strategy for the imbalanced classes problem.

\subsection{Towards a boosting method}

Let $\mathcal{H}$ be an ensemble of classifiers and $S=\lbrace(\xbf_i,y_i)\rbrace_{i=1}^m$ a learning sample.
Our goal is to find a classifier $\hat{H}$ which can be written as a weighted combination of all the classifiers of $\mathcal{H}$ and at the same time, 
which minimizes the norm of the confusion matrix.
That is:
\begin{equation}
\label{eq:boostgoal}
\hat{H} = \underset{H\in \mathcal{H}}{\operatorname{argmin}}\|\mathbf{C}_S \|^2,
\end{equation}
where $H(\cdot)=\argmax_{l\in\{1,\cdots,K\}}\sum_{h\in\mathcal{H}}\alpha_h \mathbb{I}(h(\cdot)=l)$,\ and $\alpha_h\geq0, \forall h$.

In order to make use of the bound given in equation (\ref{eq:optimizationMC}), we need to choose loss functions adapted to our problem.
We start off by defining score functions for each example $i$ and each class $l$ as follows:
\begin{equation}
\label{eq:losses}
f_H(i,l) = \sum_{h\in \mathcal{H}} \alpha_h\mathbb{I}(h(\xbf_i)=l), \forall i\in\{1,\cdots,m\},l\in\{1,\cdots,K\}.
\end{equation}

These simple functions measure the importance that the classifier $H$ gives to class $l$ for example $i$.
Since the prediction rule is based on the $\argmax$ (eq. (\ref{eq:boostgoal})), $H$ returns the class with the highest score.
If $H$ correctly classifies an example, then the score given to its true class is higher than the score given to any of the other classes.
Hence, let the loss functions be the difference between the score given to a class and the one given to the true class:

\begin{equation}
\label{eq:theloss}
\ell_{y_i,j}(H,\xbf_i) = f_H(i,j) - f_H(i,y_i), \forall j\neq y_i.
\end{equation}

If an example is correctly classified then the loss is negative, while if it is misclassified, then the loss is positive for at least one of the classes.
This observation reveals to be the main downside of these losses, since it implies that the requirements for the bound in equation (\ref{eq:optimizationMC}) might not be met.
This is why we propose to consider the exponential losses instead, and to rewrite our goal (\ref{eq:boostgoal}) as:
\begin{equation}
\label{eq:hhat}
\hat{H} = \underset{H\in \mathcal{H}}{\operatorname{argmin}}\underset{i=1}{\overset{m}{\sum}} \underset{j\neq y_i}{\sum}\frac{1}{m_{y_i}} \exp\Big(f_H(i,j) - f_H(i,y_i)\Big).
\end{equation}

Finding the optimal solution for the last equation can be quite challenging, since it depends on the size of $\mathcal{H}$ and the weights.
A popular approach when dealing with ensembles of classifiers is to use an iterative greedy method based on the boosting framework.
The goal of such methods is to consider, at each iteration, the best classifier $h\in\mathcal{H}$ --- and its weight $\alpha$ --- that minimizes the loss.
The main advantage of boosting methods resides in the fact that the selected classifier $h$, and generally all the classifiers in $\mathcal{H}$, only need to perform slightly better than random guessing, a notion that is formalized later on.
They are also known as weak classifiers.

Suppose that at iteration $t$, the classifiers $h_1,\hdots,h_t\in\mathcal{H}$ have been chosen in order to minimize the loss of $H_t(\cdot)=\argmax_{l\in\{1,\cdots,K\}}\sum_{s=1\cdots t}\alpha_s \mathbb{I}(h_s(\cdot)=l)$.
In the following iteration, the chosen classifier would be the one verifying:
\begin{equation}
\label{eq:comboMC}
h_{t+1},\alpha_{t+1} = \underset{h\in \mathcal{H},\alpha}{\operatorname{argmin}}\underset{i=1}{\overset{m}{\sum}} \underset{j\neq y_i}{\sum}\frac{1}{m_{y_i}} \exp\Big(f_{t+1}(i,j) - f_{t+1}(i,y_i)\Big),
\end{equation}
where $f_{t+1}(i,j) = \sum_{s=1\cdots t} \alpha_s\mathbb{I}(h_s(\xbf_i)=j) + \alpha\mathbb{I}(h(\xbf_i)=j)$.

This last optimization problem is quite similar to the one defined for AdaBoost.MM by \cite{Mukherjee2010}.
As such, the method proposed in this paper is an extension of AdaBoost.MM to the imbalanced classes setting.

In order to take advantage of the boosting framework recalled in section \ref{sec:notations}, we need to define a cost matrix $\mathbf{D}$ so that $\forall i$, and $\forall l\neq y_i, \mathbf{D}(i,l) \leq \mathbf{D}(i,y_i)$.
Due to the similarity of our minimization problem (\ref{eq:comboMC}) and AdaBoost.MM's, the most straightforward choice for $\mathbf{D} $ is the following :

\begin{align}
\label{def:matcost}
\mathbf{D}_t(i,l)\! &\stackrel{\hbox{\scriptsize def}}{=} \!\left\{
\begin{array}{ll}
\frac{1}{m_{y_i}}\exp(f_{t}(i,l)-f_{t}(i,y_i))&\textrm{if } l\neq y_i\\
&\\
-\underset{j\neq y_i}{\sum}\frac{1}{m_{y_i}}\exp(f_{t}(i,j)-f_{t}(i,y_i))& \textrm{otherwise}.
\end{array}\right.
\end{align}

%% file: theory.tex
\section{Algorithm and theoretical properties}
\label{sec:theory}
In the previous section, we showed why it is interesting to consider a boosting-based method for minimizing the norm of the empirical confusion matrix.
This section introduces CoMBo, the actual algorithm that performs this minimization, 
where the loss decreases at each step of the algorithm.

\subsection{The Confusion Matrix Boosting Algorithm}

\begin{algorithm}[b!]
   \caption{CoMBo : Confusion Matrix BOosting}
   \label{alg:combo}
\begin{algorithmic}
\STATE{\bf Given}
\STATE \begin{itemize}
\item S = $\lbrace (\xbf_1,y_1),...,(\xbf_m,y_m)\rbrace$ where $\xbf_i \in X$, $y_i \in \lbrace 1,\cdots,K\rbrace $
\item T the number of iterations, $\mathcal{W}$ a weak learner
\item $\forall i \in\lbrace 1, \ldots , m\rbrace$,  $\forall l\in\lbrace 1, \ldots , K\rbrace$ $f_{1}(i,l) = 0$
\item $\mathbf{D}_{1}(i,l)=\left\{
\begin{array}{lr}
   \frac{1}{m_{y_i}} & \mbox{ if } y_i\neq l \\
   \frac{-(K-1)}{m_{y_i}}& \mbox{ if } y_i = l \\
\end{array}
\right.$
\end{itemize}
 \FOR{$t=1$ {\bf to} $T$} 
\STATE Use $\mathcal{W}$ to learn $h_{t}$ with edge $\delta_{t}$ on $\mathbf{D}_{t}$, and $\alpha_{t} = \frac{1}{2}\ln\frac{1+\delta_{t}}{1-\delta_{t}}$
\STATE where :
$$
\delta_t = \frac{-\sum_{i=1}^m \mathbf{D}_{t}(i,h_t(\xbf_i))}{\sum_{i=1}^m \sum_{l\neq y_i}\mathbf{D}_{t}(i,l)}
$$
\STATE Update $\mathbf{D}$ :
\STATE $$\mathbf{D}_{t+1}(i,l)=\left\{
\begin{array}{lr}
  \frac{1}{m_{y_i}} \exp({f_{t+1}(i,l)-f_{t+1}(i,y_i)})& \mbox{ if } l\neq y_i \\
   -\frac{1}{m_{y_i}}\overset{k}{\underset{j\neq y_i}{\sum}} \exp({f_{t+1}(i,j)-f_{t+1}(i,y_i)})& \mbox{ if } l = y_i \\
\end{array}
\right.$$
\STATE \hspace{1cm}where $f_{t+1}(i,l) = \overset{t}{\underset{z=1}{\sum}} \mathbb{I} [h_{z}(i)=l]\alpha_{z}$
\ENDFOR
\STATE Output final hypothesis :
$$ H(\xbf) = \underset{l\in{1,...,k}}{\mbox{argmax}}f_T(\xbf,l),\mbox{ where }\ f_T(\xbf,l) = \overset{T}{\underset{t=1}{\sum}} \mathbb{I} [h_{t}(\xbf)=l]\alpha_{t}$$
\end{algorithmic}
\end{algorithm}

The pseudo-code of the proposed method named CoMBo is given in Algorithm \ref{alg:combo}.
The inputs : a training sample $S$, the total number of iterations $T$ and a weak learner $\mathcal{W}$.
During the initialization step, the score functions $f$ are set to zero and the cost matrix $\mathbf{D}$ is initialized accordingly.

The training phase consists of two steps: the weak learner $\mathcal{W}$ is used in order to build the weak classifiers, and the predictions of $h_t$ are used to update the cost matrix $\mathbf{D}_t$.
At each round $t$, $\mathcal{W}$ takes as input the cost matrix $\mathbf{D}_t$ and returns a weak classifier $h_t$.
The cost matrix is then used to compute the weight $\alpha_t$ for $h_t$, which can be seen as the importance given to $h_t$.
$\alpha_t$ depends on the edge $\delta_t$ obtained by $h_t$ over the cost matrix $\mathbf{D}_t$.
For boosting methods based on AdaBoost, the edge measures the difference of performances of the classifier and of random guessing \citep{Mukherjee2010}.
The underlying idea is that the better $h_t$ performs over $\mathbf{D}_t$, the greater the edge $\delta_t$ and $\alpha_t$.

The update rule for the cost matrix is designed so that the misclassification cost is increased for the examples that are misclassified by $h_t$, 
while it is decreased for the correctly classified ones.
This forces the weak learner $\mathcal{W}$ to focus on the most difficult examples.
Then, using the term $1/m_{y_i}$ in the update rule, allows the misclassification cost of an example 
to depend not only on the ability to correctly classify a hard example, as in AdaBoost.MM, but also on the number of examples of $S$ having the same class $y_i$.

The output hypothesis is a simple weighted majority vote over the whole set of weak classifiers.
So, for a given example, the final prediction is the class that obtains the highest score.

\subsection{Bounding the loss}

Let us recall the minimal weak learning condition as given by \cite{Mukherjee2010}.

\begin{definition}
Let $D^{eor}$ be the space of all cost matrices $\mathbf{D}$ that put the least cost on the correct label, that is $\forall (\xbf_i,y_i),l, \mathbf{D}(i,y_i) \leq \mathbf{D}(i,l)$.
Let $B_\gamma^{eor}$ be the space of baselines $\mathbf{B}$ which are $\gamma$ more likely to predict the correct label for every example $(\xbf_i,y_i)$, i.e. $\forall l\neq y_i, \mathbf{B}(i,y_i) \geq \mathbf{B}(i,l)+\gamma$.
Then, the minimal weak learning condition is given by:
\begin{align}
\label{def:mwlc}
\forall \mathbf{D}\in D^{eor}, \exists h\in \mathcal{H} : \mathbf{D} \cdot \mathbf{1}_h \leq \underset{\mathbf{B}\in B^{eor}}{\max} \mathbf{D}\cdot \mathbf{B},
\end{align}
where $\mathcal{H}$ is a classifier space, and $\mathbf{1}_h$ is the matrix defined as  $\mathbf{1}_h(i,l)=\mathbb{I}[h(i)=l]$.
\end{definition}

In the rest of this paper, we consider baselines that are close to uniform random guessing.
Noted $\mathbf{U}_\gamma$, these baselines have weights $(1-\gamma)/k$ on incorrect labels and $(1-\gamma)/k+\gamma$ on correct ones.
The weak learning condition is given by :
\begin{align}
\label{eq:unicond}
\mathbf{D} \cdot \mathbf{1}_h \leq \mathbf{D}\cdot \mathbf{U}_\gamma
\end{align}
All of the weak classifiers returned by $\mathcal{W}$ during the training phase are assumed to verify that weak learner condition.

Lemma \ref{le:drop} shows that the general loss decreases with each iteration, provided that the weak classifier $h_t$ satisfies the weak learning condition (\ref{eq:unicond}).
This result and its proof are similar to the ones given for AdaBoost.MM.

\begin{lemma}
\label{le:drop}
Suppose the cost matrix $\mathbf{D}_{t}$ is chosen as in the algorithm \ref{alg:combo}, 
and the returned classifier $h_{t}$ satisfies the edge condition for the baseline $\mathbf{U}_{\delta_t}$ and cost matrix $\mathbf{D}_{t}$,
{\em i.e.} $\mathbf{D}_{t} \cdot \mathbf{1}_{h_{t}} \leq \mathbf{D}_{t} \cdot \mathbf{U}_{\delta_t}$. 

Then choosing  a weight $\alpha_{t} > 0$ for $h_{t}$ makes the loss at round $t$ 
at most a factor 
$$ 1- \frac{1}{2} \left ( e^{\alpha_{t}} \right . -e^{-\alpha_{t}})\delta_{t}+\frac{1}{2} \left (e^{\alpha_{t}}+e^{-\alpha_{t}}-2) \right .$$
of the loss before choosing $\alpha_{t}$, where $\delta_{t}$ is the edge of $h_{t}$.
\end{lemma}

\begin{proof}
For readability reasons, let us note $L_t$ and $L_{t}(i)$ the following losses:
$$
L_t = \sum_{i=1}^m L_t(i), \text{ where }
L_t(i) = \sum_{l\neq y_i}\frac{1}{m_{y_i}}\exp({f_{t}(i,l)-f_{t}(i,y_i)})
$$
The weak classifier $h_t$ returned by $\mathcal{W}$ satisfies the edge condition, that is:
\begin{align}
\label{eq:edgecond}
\mathbf{D_t} \cdot \mathbf{1}_{h_t} \leq \mathbf{D_t}\cdot \mathbf{U}_{\delta_t},
\end{align}
with $\delta_t$ being the edge of $h_t$ on $\mathbf{D_t}$.

Denote $S_+$ (resp. $S_-$) the set of examples of $S$ correctly classified (resp. misclassified) by $h_t$.
Using the definitions of $\mathbf{D_t}$, $ \mathbf{1}_{h_t} $ and $\mathbf{U}_{\delta_t}$, the classification costs of $h_t$ and   $\mathbf{U}_{\delta_t}$ are given by:
$$
\begin{array}{lll}
\mathbf{D}_{t}  \cdot \mathbf{1}_{h_{t}} &= & -\underset{i\in S_+}{\sum}L_{t-1}(i) + \underset{i \in S_-}{\sum}\frac{1}{m_{y_i}}{\exp\left(f_{t-1}(i,h_{t}(i)) - f_{t-1}(i,y_i)\right)} = -A_+^t + A_-^t
\\
\mathbf{D}_{t} \cdot \mathbf{U}_{\delta_{t}}  &= & -\delta_{t} \sum_{i=1}^m L_{t-1}(i) = -\delta_{t} L_{t-1}
\end{array}
$$
Inserting these two costs in (\ref{eq:edgecond}), we have :
\begin{align}
\label{eq:apam}
A_+^t - A_-^t \geq \delta_{t} L_{t-1}.
\end{align}
Taking a closer look at the drop of the loss after choosing $h_t$ and $\alpha_t$, we have:
$$
\begin{array}{rl}
L_{t-1} - L_{t} &=   \underset{i \in S_+}{\sum}L_{t-1}(i)(1-e^{-\alpha_t}) +  \underset{i \in S_-}{\sum}\frac{1}{m_{y_i}}{\exp(\Delta f_{t-1}(i,h_{t},y_i))}(1-e^{\alpha_t})\\
& = (1-e^{-\alpha_t})A_+^t - (e^{\alpha_t}-1)A_-^t\\
& = \left( \frac{e^{\alpha_t}-e^{-\alpha_t}}{2}\right)(A_+^t - A_-^t) - \left( \frac{e^{\alpha_t}+e^{-\alpha_t}-2}{2}\right)(A_+^t + A_-^t)
\end{array}
$$
where $\Delta f_{t-1}(i,h_{t},y_i)) = f_{t-1}(i,h_t(i))-f_{t-1}(i,y_i)$.

The result in (\ref{eq:apam}) gives a lower bound for $A_+^t - A_-^t$, while $A_+^t + A_-^t$ is upper-bounded by $L_{t-1}$.
Hence,
$$
L_{t-1} - L_{t} =   \underset{i \in S_+}{\sum}L_{t-1}(i)(1-e^{-\alpha_t}) \geq ( \frac{e^{\alpha_t}-e^{-\alpha_t}}{2})\delta_t L_{t-1} - ( \frac{e^{\alpha_t}+e^{-\alpha_t}-2}{2})L_{t-1}.
$$
Therefore, the result of the lemma:
\begin{equation}
\label{eq:lossdrop}
\begin{array}{rl}
L_{t} \leq \left(1 - \frac{e^{\alpha_t}-e^{-\alpha_t}}{2}\delta_t +  \frac{e^{\alpha_t}+e^{-\alpha_t}-2}{2}\right)L_{t-1} =  \left(\frac{((1-\delta_t)}{2} e^{\alpha_t}+\frac{(1+\delta_t)}{2}e^{-\alpha_t})\right)L_{t-1}.
\end{array}
\end{equation}
\end{proof}

The expression of the loss drop given in Lemma \ref{le:drop} can be further simplified.
Indeed, if we choose the value of $\alpha_t$ as given in the pseudo-code of Algorithm \ref{alg:combo}, then the loss drop is simply equal to $\sqrt{1 - \delta_t^2}$.
Since the value of $\delta_t^2$ is always positive, $\sqrt{1 - \delta_t^2}$ is smaller than 1, thus the loss $L_{t}$ is always smaller than $L_{t-1}$.
The following theorem summarizes this result.

\begin{theorem}
Let $\delta_1,\cdots,\delta_T$ be the edges of the classifiers $h_1,\cdots, h_T$ returned by $\mathcal{W}$ at each round of the learning phase.
Then the error after $T$ rounds is $K(K-1)\prod_{t=1}^T\sqrt{1-\delta_t^2}\leq K(K-1)\exp\left\{ -(1/2)\sum_{t=1}^T\delta_t^2\right\}$.

Moreover, if there exists a $\gamma$ so that $\forall t, \delta_t\geq \gamma$, then the error after $T$ rounds is exponentially small, $K(K-1)e^{-T\gamma^2/2}$.
\end{theorem}

%% file: results.tex
\section{Experimental results}
\label{sec:expe}

\subsection{Datasets and experimental setup}

In order to compare CoMBo with other approaches, it was run on 9 imbalanced multi-class classification datasets. 
They are all from the UCI Machine Learning Repository \citep{Frank10uci}.
They exhibit various degrees of imbalance, as well as various numbers of instances and attributes. 
Table \ref{table:distribution} summarizes information about these datasets.

\begin{table}[h]
\centering
\begin{tabular}{|c|c|c|c|c|}
\hline
Dataset & \# classes & \# examples & distribution of classes & Imb. ratio \\
 \hline
New-Thyroid & 3 & 215 & 150/35/30 & 5.00 \\
\hline
Balance & 3 & 625 & 49/288/288 & 5.88 \\
\hline
Car & 4 & 1728 & 1210/384/69/65 & 18.62\\
\hline
Connect & 3 & 67557 & 44473/6449/16635 & 6.90 \\
\hline
Nursery-s & 4 & 12958 & 4320/328/4266/4044 & 13.17 \\
\hline
Glass & 6 & 214 & 70/76/17/13/9/29 & 8.44 \\
\hline
E.coli & 5 & 327 & 143/77/52/35/20 & 7.15 \\
\hline
Yeast & 10 & 1484 & 463/429/244/163/51/44/35/30/20/5 & 92.60\\
\hline
Satimage & 6 & 6435 & 1533/703/1358/626/707/1508 & 2.45 \\
\hline
\end{tabular}
\caption{Class distributions of considered UCI datasets (the last column reports the ratio between the \# of instances of the majority class and the \# of instances of the 
minority class (imbalance ratio).)}  
\label{table:distribution}
\end{table}

For each dataset, we performed 10 runs of 5-fold cross-validation of CoMBo and Adaboost.MM with a low-size decision tree as a base learner ($\mathcal{W}$) and $T=200$, and we averaged the results. 
Two kinds of results are reported: advanced accuracy performances are first compiled (MAUC, Gmean), then the behavior of CoMBo is compared to AdaBoost.MM with a deeper insight on
the norm of the confusion matrix.

\subsection{Performance results}
Usual performance measures of multi-class classification include MAUC (an extension of the AUC to multi-class problems, \citep{Hand2001simple}), and 
G-mean (the geometric mean of recall values over all classes \citep{Sun2006boosting}).

Table \ref{table:resultsmc} reports these measures for CoMBo, as well as results from two boosting-based algorithms relying on resampling: AdaBoost.NC with oversampling \citep{Wang2012multiclass}, 
and SmoteBoost \citep{Chawla2003smoteboost}.
These algorithms were chosen for comparison because they showed the best results on the considered datasets \citep{Wang2012multiclass}.

\begin{table}[t]
\begin{center}
\begin{tabular}{|l|c|c|c|}
 \hline
{\bf G-mean}  & CoMBo & AdaBoost.NC & SmoteBoost  \\
\hline
Car & $\mathbf{0.967} \pm 0.023$ & $0.924 \pm 0.024$ & $0.944 \pm 0.031$  \\
Balance &  $\mathbf{0.675} \pm 0.088$ & $0.321 \pm 0.173$ & $0.000 \pm 0.000$  \\
New-Thyroid & $0.914 \pm 0.081$ & $0.927 \pm 0.056$ & $0.940 \pm 0.057$   \\
Nursery-s & $\mathbf{1.000} \pm 0.001$ & $0.967 \pm 0.006$ & $0.996 \pm 0.003$  \\
Ecoli & $0.784 \pm 0.066$ & $0.790 \pm 0.062$ & $0.803 \pm 0.059$  \\
Glass & $0.431 \pm 0.0.375$ & $0.578 \pm 0.249$ & $0.561 \pm 0.343$  \\
Satimage & $\dagger 0.825 \pm 0.012$ & $0.881 \pm 0.009$ & $\mathbf{0.898} \pm 0.010 $ \\
Yeast & $0.107 \pm 0.216$ & $0.237 \pm 0.270$ & $0.140 \pm 0.240$  \\
\hline
\hline
{\bf MAUC} & CoMBo & AdaBoost.NC & SmoteBoost  \\
\hline
Car & $0.993 \pm 0.003$ & $0.982 \pm 0.005$ & $\mathbf{0.997 \pm 0.000}$ \\
Balance & $\mathbf{0.884} \pm 0.032$ & $0.704 \pm 0.037$ & $0.703 \pm 0.027$ \\
New-Thyroid & $\mathbf{0.996} \pm 0.010$ & $0.983 \pm 0.013$ & $0.988 \pm 0.003$ \\
Nursery-s & $\mathbf{1.000} \pm 0.000$ & $0.998 \pm 0.000$ & $0.999 \pm 0.000$ \\
Ecoli & $0.961 \pm 0.015$ & $0.957 \pm 0.002$ & $0.963 \pm 0.004$ \\
Glass & $\mathbf{0.947} \pm 0.027$ & $0.881 \pm 0.009$ & $0.925 \pm 0.009$ \\
Satimage & $\dagger 0.976 \pm 0.003$ & $0.990 \pm 0.000$ & $\mathbf{0.992} \pm 0.000$ \\
Yeast & $0.861 \pm 0.025$ & $0.857 \pm 0.004$ & $0.847 \pm 0.003$ \\
\hline
\end{tabular}
\caption{Comparison with algorithms addressing classification within imbalanced datasets. 
Means and standard deviations are given over 10 runs of 5-fold cross-validations. 
Except for CoMBo, the results are the {\em best} retrieved from \cite{Wang2012multiclass}, which analysed four algorithms. 
Results in boldface indicate a significantly best measure of the algorithm over all the others, according to the sudent T-test with a confidence of 95\%; 
the $\dagger$ attached to the result on {\tt Satimage} indicates that CoMBo is significantly worse than all other reported methods.}  
\label{table:resultsmc} 
\end{center}
\end{table}

Looking at results in Table \ref{table:resultsmc}, CoMBo is quite promising w.r.t. current literature, without any tuning nor trying to optimize the reported measures.
Only the best results of AdaBoost.NC and SmoteBoost from \cite{Wang2012multiclass} are reported.
Except for the dataset {\tt Satimage} (the less imbalanced dataset), CoMBo challenges the other boosting-based approaches. 
The advantage of preferring CoMBo over the reported methods seems to be related to the imbalance ratio: the higher this ratio is, the more CoMBo makes the difference,
which is consistent with the aim of CoMBo.
It is worth noticing that CoMBo does not need any prior initialisation of the 
misclassification cost matrix: it computes and adjusts the cost matrix along the learning stages, based on the confusion matrix norm computed at each step.
In other words, CoMBo gets these good results without any other parameters than the weak algorithm (actually a 2 to 3-depth decision tree), and the number $T$ of rounds.

\subsection{Improvements from Adaboost.MM}
One may think that the good results of CoMBo are partly due to the fact that it is based on the theory underlying Adaboost.MM.
Therefore, it is worth exploring the way CoMBo differs from Adaboost.MM in the processing of imbalanced datasets through the norm of the confusion matrix.
Results are presented Table \ref{table:results}: the confusion matrix norms are reported, together with 
the accuracies, the MAUC and the G-mean.

\begin{table}[t]
\centering
\begin{tabular}{|l||c|c||c|c|}

\hline
dataset & CoMBo & AdaMM & CoMBo & AdaMM  \\
&    $\|\mathbf{C}\|$ & $\|\mathbf{C}\|$ & accuracy & accuracy  \\
\hline
Balance & $\mathbf{0.460} \pm 0.097$ & $0.559 \pm 0.112$ & $0.856 \pm 0.039$ & $0.875 \pm 0.032$ \\
\hline
Car         & $0.082 \pm 0.064$ & $0.116 \pm 0.080$ & $0.974 \pm 0.009$ & $\mathbf{0.977} \pm 0.012$  \\
\hline 
Connect & $\mathbf{0.308} \pm 0.015$ & $0.670 \pm 0.040$ & $0.728 \pm 0.015$ & $\mathbf{0.805} \pm 0.009$ \\
\hline
New-Thyroid & $0.194 \pm 0.158$ & $0.186 \pm 0.157$& $0.949 \pm 0.033$ & $0.952 \pm 0.033$  \\
\hline
Nursery-s   & $0.002 \pm 0.004$ & $0.003 \pm 0.006$ & $1.000 \pm 0.000$ & $1.000 \pm 0.001$  \\
\hline 
Yeast       & $\mathbf{0.815} \pm 0.150$ & $1.101 \pm 0.194$ & $0.572 \pm 0.025$ & $0.577 \pm 0.025$  \\
\hline 
\hline
& G-mean & G-mean & MAUC & MAUC \\
\hline
Balance & $\mathbf{0.675} \pm 0.088$ & $0.566 \pm 0.190$ & $\mathbf{0.884} \pm 0.032$ & $0.855 \pm 0.039$ \\
\hline 
Car         &  $0.967 \pm 0.023$ & $0.954 \pm 0.034$ &  $\mathbf{0.993} \pm 0.003$ & $0.988 \pm 0.004$ \\
\hline 
Connect & $\mathbf{0.703} \pm 0.013$ & $0.497 \pm 0.035$ &  $\mathbf{0.863} \pm 0.008$ & $0.852 \pm 0.009$ \\
\hline
New-Thyroid & $0.914 \pm 0.081$ & $0.915 \pm 0.075$ & $0.996 \pm 0.010$ & $0.996 \pm 0.005$ \\
\hline
Nursery-s   & $\mathbf{1.000} \pm 0.001$ & $0.999 \pm 0.002$ & $1.000 \pm 0.000$ & $1.000 \pm 0.000$ \\
\hline
Yeast       & $0.107 \pm 0.216$ & $0.000 \pm 0.000$ & $\mathbf{0.861} \pm 0.025$ & $0.847 \pm 0.003$ \\
\hline
\end{tabular}
\caption{Adaboost.MM vs. CoMBo. In non-reported datasets, results of both algorithms are equivalent). 
Means and standard deviations are given over 10 runs of 5-fold cross-validations. 
Results in boldface indicate a significantly best measure according to the sudent T-test with a confidence of 95\%.}
\label{table:results}
\end{table}

These results confirm the good results of CoMBo on imbalanced datasets as evidenced by the low values of the norms. 
They illustrate the impact of minimizing the norm of the confusion matrix.
Let us note that the accuracy with CoMBO tends to be a bit worse than the one of Adaboost.MM. 
Meanwhile, as expected, the norm of the confusion matrix is always smaller with CoMBo. 
The performances of the classifier trained with CoMBo is smoothed throughout all the classes, whatever the number of examples they feature in the dataset. 
That way, majority classes are not as favored as they usually are in multi-class approaches. 

On non-reported datasets, computed measures for both algorithms are very close. On reported results, we observe that
the higher the decrease of the confusion matrix norm, the higher the gain on MAUC and G-mean.
However, we still have to investigate the relation between MAUC/G-mean and the norm of the confusion matrix. 
These preliminary results let us think that there might be no gain using CoMBo instead 
of Adaboost.MM in other cases, but there is no loss either (the computational times are the same).

\begin{table}[htb]
\begin{center}
\begin{tabular}{cc}
$\begin{pmatrix}
0.000 & 0.147 & 0.068 \\
0.146 & 0.000 & 0.169 \\
0.056 & 0.203 & 0.000
 \end{pmatrix}$
 &
$\begin{pmatrix}
0.000 & 0.011 & 0.055 \\
0.656 & 0.000 & 0.179 \\
0.246 & 0.048 & 0.000 
\end{pmatrix}$ \\
 CoMBo  & Adaboost.MM 
\end{tabular}
\caption{Confusion matrices obtained with Adaboost.MM and CoMBo, on the dataset {\tt Connect}.}
\label{fig:inside}
\end{center}
\end{table}

Finally, Table \ref{fig:inside} illustrates what actually occurs on the dataset {\tt Connect}\footnote{The dataset {\tt Connect} was chosen for illustration because
of readability (it only features 3 classes) and because it is much imbalanced.}. 
The errors on minority classes 2 (83.5\%) and 3 (29.4\%) are dramatically high with Adaboost.MM which 
promotes the majority class (only 6.6\% of errors). These differences are reduced with CoMBo: the error on the majority
class reaches 20.2\% while errors on minority classes decrease respectively to 31.5\% and 25.9\%.
However, the real error is still higher with CoMBo: 
misclassified examples of the majority classes getting more numerous, it directly impacts the overall error rate.  
Such a behavior of CoMBo points out that it equally considers each class during the learning process, independently from any
tricky misclassification cost.

These experiments acknowledge the smoothed learning processed by CoMBo over imbalanced classes.
We think that this smoothing effect is related to the link between the minimization of the confusion matrix's norm and the MAUC/G-mean, since it allows CoMBo to perform well on all the classes, contrary to AdaBoost.MM, who performs excellently on majority classes and poorly on minority ones.

%% file: conclusion.tex
\section{Discussion and future works}
\label{sec:discussion}

The framework proposed in this paper raises several questions and prospect works, some of which are discussed in this section.

Some potential extensions consist in using the result obtained in equation (\ref{eq:optimizationMC}) in order to derive other cost sensitive algorithms.
As briefly discussed in section \ref{sec:combo}, the optimization term depends on the different loss functions defined over the classes.
The only restriction on these losses is that their value should be greater than the one returned by the indicator function.
This implies that it is possible to embed various informations into the loss functions, as long as they respect this condition.
For instance, it is easy to include in the proposed method penalization terms defined over two classes simply by replacing the loss functions considered in section \ref{sec:combo} with $\ell_{y_i,j}(h,\xbf_i)=c_{y_i,j}\exp(f_H(i,j)-f_H(i,y_i))$, where $c_{y_i,j}$ is a penalization term given as a prior or coming from a cost matrix.

Motivated by the theoretical and experimental results --- think of the difference between the estimated error of the proposed method and AdaBoost.MM's --- other works should focus on finding tighter bounds for the operator norm of the confusion matrix.
Depending on the singular values computed from the confusion matrix, the bound given in equation (\ref{eq:optimizationMC}) can be quite loose.
Finding a tighter bound would then be a first step towards a more effective method in minimizing both  the risk and the norm of the confusion matrix alike.

Continuing on the same line of thought, an important question that arises from the proposed framework is: {\em is it possible to obtain similar results for other norms}?
Even though minimizing the different norms of the confusion matrix may be the same thing due to the equivalence between the norms, the main difference consists in whether it is possible to find an analytical expression of the norm.
Some of the considered norms can be the $l1-$norm, entry-wise norms or even more exotic ones.

A second question that naturally comes to mind is: {\em is it possible to consider other definitions for the confusion matrix, such as loss-based matrices}, 
where the entries are replaced by loss functions.
More generally, it would be interesting to see if there exists a relation between the choice of the confusion matrix and choice of the norm to be minimized.

Finally, based on the empirical results, we think that the norm of the confusion matrix is quite an useful tool for measuring the performance of a model in the multi-class and/or imbalanced classes setting, alongside classical measures as MAUC, G-mean, F-measure, etc.

\section{Conclusion}
\label{sec:conclusion}

The contributions of this paper are two-fold: first we make use of the norm of the confusion matrix as an error measure in order to derive a common bound for most cost sensitive methods; secondly, using the bound as a starting point, we show step by step how to obtain an extension of AdaBoost.MM for the imbalanced classes framework.
Empirical results show that the proposed method compares favorably to other cost-based boosting methods.
These performances are both due to the inherited boosting framework --- which is more adapted to the multi-class problems than other frameworks --- 
and to the weighting scheme obtained from the choice of the loss functions in section \ref{sec:combo}.
Although much work remains to be done, as discussed in section \ref{sec:discussion}, the proposed method is, to the best of our knowledge, the first boosting based approach that aims to actively minimize the norm of the confusion matrix, by including it in the learning process.

\section*{Acknowledgements}

The authors would like to thank Liva Ralaivola for his inputs and suggestions on the work presented in this paper.